\pdfoutput=1 % Comment in for arXiv version
\documentclass[final,12pt]{clear2023} % Include author names

\title[Practical Algorithms for Orientations of Partially Directed Graphical
Models]{Practical Algorithms for Orientations of\\ Partially Directed Graphical
Models}
\usepackage{times}
\usepackage{tikz}
\usetikzlibrary{arrows.meta, graphs, shapes, decorations.pathreplacing, calligraphy}
\usepackage{booktabs}

\usepackage{pifont}% http://ctan.org/pkg/pifont
\newcommand{\cmark}{\ding{51}}%
\newcommand{\xmark}{\ding{55}}%
\tikzset{
  pico/.style = {
    every node/.style = {
      draw,
      circle,
      semithick,
      inner sep = 0pt,
      minimum width = 0.7ex,
      fill = white
    },
    semithick
  },
    edge/.style = {
    semithick
  },
  arc/.style = {
    ->,
    semithick,
    >={[round,sep]Stealth}
  }
}

\newtheorem{thm}{Theorem}
\newtheorem{defi}{Definition}
\newtheorem{problem}{Problem}
\newtheorem{fact}{Fact}

\newcommand{\Pa}{\mathrm{Pa}} 
\newcommand{\Ch}{\mathrm{Ch}} 
\newcommand{\Ne}{\mathrm{Ne}} 
\newcommand{\Si}{\mathrm{Si}} 

\newcommand{\isadj}{\overset{?}{\sim}} % make this operator? 

\renewcommand\rightarrow[1][1.4em]{\tikz[baseline=-0.5ex, shorten
  <=2pt, shorten >=2pt] \draw[-{Stealth[round,sep]}] (0,0) -- (#1,0);}
\renewcommand\leftarrow[1][1.4em]{\tikz[baseline=-0.5ex, shorten
  <=2pt, shorten >=2pt] \draw[{Stealth[round,sep]}-] (0,0) -- (#1,0);}

\newenvironment{computationalproblem}%
{%
  \leavevmode\nobreak\par
  \begin{list}%
    {}%
    {%
      \def\labelstyle{\itshape}
      \setlength{\topsep}{0pt}%
      \settowidth{\labelwidth}{\labelstyle Output:}%
      \setlength{\leftmargin}{\labelwidth}%
      \addtolength{\leftmargin}{\labelsep}%
      \setlength{\itemsep}{0pt}%
      \setlength{\parsep}{0pt}%
    }%
      \def\input{\item[\labelstyle Input:]}%
      \def\output{\item[\labelstyle Output:]}%
    }%
    {%
  \end{list}%
}

\clearauthor{%
 \Name{Malte Luttermann}$^\ast$ \Email{luttermann@ifis.uni-luebeck.de}\\
 \addr Institute of Information Systems, University of Lübeck, Germany
 \AND
 \Name{Marcel Wienöbst}$^\ast$ \Email{wienoebst@tcs.uni-luebeck.de}\\
 \addr Institute for Theoretical Computer Science, University of Lübeck, Germany
 \AND
 \Name{Maciej Li\'{s}kiewicz} \Email{liskiewi@tcs.uni-luebeck.de}\\
 \addr Institute for Theoretical Computer Science, University of Lübeck, Germany
}

\begin{document}

\maketitle

\begin{abstract}%
% TODO: observational x2
In observational studies, the true causal model is typically unknown 
and needs to be estimated from available observational and limited experimental data.
In such cases, the learned causal model is commonly represented as a partially 
directed acyclic graph (PDAG), which contains both directed and undirected 
edges indicating uncertainty of causal relations between random variables. 
The main focus of this paper is on the maximal orientation task, which, 
for a given PDAG, aims to orient the undirected edges maximally such that 
the resulting graph represents the same Markov equivalent DAGs as the 
input PDAG. This task is a subroutine used frequently in causal discovery,
e.\,g., as the final step of the celebrated PC algorithm.
Utilizing connections to the problem of finding a consistent DAG extension of a PDAG,
we derive faster algorithms for computing the maximal
orientation by proposing two novel approaches 
for extending PDAGs, both constructed with an emphasis on simplicity and practical effectiveness.
\end{abstract}


\def\thefootnote{$^\ast$}
\def\footnoteseptext{}
\footnotetext{Equal contribution.}
\def\thefootnote{\arabic{footnote}}
\def\footnoteseptext{. }

\section{Introduction} \label{sec:intro}
The development of probabilistic graphical models enables a mathematically sound 
language to handle uncertainty in a coherent and compact way~\citep{Spirtes2000,pearl2009,Koller2009,Elwert2013}.
They also provide scientists with intuitive tools for causal analysis and currently 
receive substantial attention in epidemiology, sociology and other disciplines. 
Moreover, the graphical modeling approach allows for the use of computational 
methods that have enabled significant progress towards 
automated causal inference and causal structure discovery.

Certainly, one of the most prominent graphical models is the 
\emph{directed acyclic graph} (DAG), whose edges
encode direct causal influences between the random variables of interest.
In practice, however, the underlying true DAGs are unknown 
and from observational or limited experimental data they can only be inferred
to a certain degree of uncertainty. In such cases, the learned causal model is 
given, typically, as a \emph{partially directed acyclic graph} (PDAG)
which contains both directed and undirected edges. Such a PDAG represents 
a class of \emph{Markov equivalent} DAGs encoding the same statistical
properties and its undirected edges indicate which directed edges 
may vary across DAGs of the class~\citep{verma1990,Meek1995,Andersson1997}. 

In this paper, we investigate \emph{orientations} of PDAG models -- 
primitive tasks used to solve more complex problems of causal analysis.
Our goal is to provide simple and effective algorithms that improve performance 
of downstream tasks, e.g., in causal structure 
learning, active learning or causal effect estimation.
The main focus of our study is on the \emph{maximal orientation} problem, the
goal of which is to orient a
maximal number of undirected edges in a PDAG $G$ such that the resulting graph, called
\emph{maximally oriented PDAG} (MPDAG), represents the same DAGs as $G$. 
It is well known that the MPDAG can be obtained from the PDAG by closing it under the
celebrated Meek rules \citep{Meek1995}. For an example PDAG $G$ and its MPDAG
$M$, see Fig.~\ref{fig:ce:mo:new}.

The primary significance of the MPDAG model, including \emph{completed PDAGs} (CPDAGs,
also called \emph{essential graphs})~\citep{Andersson1997}
is that it represents Markov equivalent DAGs 
in an elegant and convenient way. It is commonly used for solving many important 
problems as, e.g., counting and sampling Markov equivalent DAGs~\citep{He2015,wienobst2021counting},
estimating causal effects~\citep{maathuis2009estimating,van2016separators,perkovic2017complete}, 
or learning causal models~\citep{chickering2002optimal},
where CPDAGs represent the states of the search space.
Consequently, orienting a given PDAG maximally, is a frequently used primitive
in causal inference and discovery, perhaps
most prominently arising in constraint-based causal structure learning, as,
e.\,g., the
final step in the PC algorithm~\citep{Spirtes2000,kalisch2007estimating} and its
  modifications. 
%or the IC algorithm~\citep{pearl1991theory}.
Similarly, several score-based algorithms, e.g.,~\citep{pellet2008using} based on 
a generic feature-selection approach, rely on this task as well.
Another important example are algorithms for active learning which, beyond observational,
use experimental (interventional) data to resolve orientation ambiguities.
In searching for an optimal strategy, typical algorithms iteratively construct
\emph{interventional essential graphs}, which are obtained by orienting subsequent 
PDAGs
maximally~\citep{He2008,hauser2012characterization,shanmugam2015learning,activelearningdct2020}.
Furthermore, methods enumerating all possible total effects or estimating bounds on the effects 
in  a Markov equivalence class need a subroutine to compute the maximal 
orientations~\citep{maathuis2009estimating,guo2021minimal}. Often, the maximal orientation task is performed not only once,
but repeatedly, necessitating
efficient algorithms for this task.

While the algorithms for computing maximal orientations 
used in practice resort to directly applying the Meek rules, the best theoretical methods are based 
on two other important primitives of causal analysis: (i)~the \emph{consistent extension} of a PDAG to a DAG and  
(ii)~computing the CPDAG representation of the graphs Markov equivalent to a given DAG. This was already 
mentioned by~\cite{Chickering1995} and generalized to instances with background knowledge 
by~\cite{WBL2021}. Using the clever algorithm of~\cite{Chickering1995}, the
second task (ii) can be solved 
in linear time. On the other hand, the computational complexity of the first problem, which is to 
orient all undirected edges such that no new v-structures arises, is significantly
higher. Hence, developing effective practical methods for finding consistent
extensions of PDAGs can be a key building block in the efficient computation of maximal orientations. 
%for an example of PDAG extension, see Fig.~\ref{fig:ce:mo:new}). 
%In this paper, we analyze the performance of the state-of-the-art algorithms for consistent 
%PDAG extension by \citet{DorTarsi1992} and \citep{WBL2021} and propose 
%two novel approaches to extend PDAGs with a focus on simplicity and effectiveness.

\begin{figure}
  \centering
  \begin{tikzpicture}[yscale=0.75]
    \node (desc) at (0,.9) {$G$};
    \node (a) at (0,0) {$a$};
    \node (b) at (1,1) {$b$};
    \node (c) at (1,-1) {$c$};
    \node (d) at (2,0) {$d$};
    \node (e) at (3,0) {$e$};

    \draw [-] (a) -- (b);
    \draw [-] (a) -- (c);
    \draw [arc] (b) -- (d);
    \draw [arc] (c) -- (d);
    \draw [-] (d) -- (e);
    \draw [-] (a) -- (d);
  \end{tikzpicture}
  \hspace*{7mm}
  \begin{tikzpicture}[yscale=0.75]
    \node (desc) at (0,.9) {$M$};
    \node (a) at (0,0) {$a$};
    \node (b) at (1,1) {$b$};
    \node (c) at (1,-1) {$c$};
    \node (d) at (2,0) {$d$};
    \node (e) at (3,0) {$e$};

    \draw [-] (a) -- (b);
    \draw [-] (a) -- (c);
    \draw [arc] (b) -- (d);
    \draw [arc] (c) -- (d);
    \draw [arc] (d) -- (e);
    \draw [arc] (a) -- (d);
  \end{tikzpicture}
   \hspace*{7mm}
   \begin{tikzpicture}[yscale=0.75]
   \node (desc) at (0,.9) {$D$};
    \node (a) at (0,0) {$a$};
    \node (b) at (1,1) {$b$};
    \node (c) at (1,-1) {$c$};
    \node (d) at (2,0) {$d$};
    \node (e) at (3,0) {$e$};

    \draw [arc] (a) -- (b);
    \draw [arc] (a) -- (c);
    \draw [arc] (b) -- (d);
    \draw [arc] (c) -- (d);
    \draw [arc] (d) -- (e);
    \draw [arc] (a) -- (d);
  \end{tikzpicture}

 \caption{A PDAG $G$ with a maximal orientation $M$. DAG $D$ shows a consistent extension of $G$.}
  \label{fig:ce:mo:new}
\end{figure}
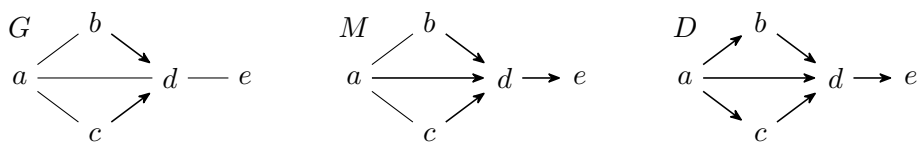

\paragraph*{Previous Work.}
The study on extendability of PDAGs has been initiated by
\cite{verma1992algorithm}, who provided a method
to find a consistent DAG  extension in time $O(n^4 m)$, where $n$ denotes the number of nodes and 
$m$  the number of edges. At the same time, \cite{DorTarsi1992} proposed a  faster method of time 
complexity $O(n^4)$ which is conceptually simple, easy to implement, and widely used in practice so far.
In \citeyear{WBL2021}, \citeauthor{WBL2021} proposed a new algorithm for the problem
which runs in time $O(n^3)$.
Simultaneously, the authors showed that, under a computational intractability 
assumption, the  cubic algorithm is optimal.

Despite its importance, relatively little attention has been paid to the algorithmic 
aspects of the maximal orientation problem so far. A commonly used approach
relies on the direct, iterative  application of the four Meek rules  
%(which are shown in Fig.~\ref{figure:meek:rules})
until none can be applied anymore.
This approach is, however, computationally expensive, and its worst-case run time
is $O(n^4 m)$, where $n$ and $m$ denote the number of vertices and edges of the
graph. In \citeyear{Chickering1995}, \citeauthor{Chickering1995}
proposed a sophisticated algorithm, avoiding the direct application of the Meek rules, which firstly 
extends a PDAG  to a DAG and next computes the corresponding CPDAG directly 
from the DAG. Since a DAG can be transformed to the corresponding CPDAG in linear time,
the total running time of the algorithm is dominated by the time
needed for extension. This approach is used in the GES
algorithm~\citep{chickering2002optimal} and in a modified way in the GIES
algorithm~\citep{hauser2012characterization}, the latter however
having
the same worst-case complexity as the direct application of the Meek
rules discussed above.
More recently, \citet{WBL2021} generalized both approaches, 
while maintaining the optimal time complexity of $O(n^3)$, to maximally orient a PDAG to an MPDAG.

Finally, in more structured cases, the maximal orientation task
may be performed in linear-time $O(n+m)$. This occurs in the setting
of active
learning using \emph{single-target} interventions (that is, only single
variables can be manipulated at a time) and is shown
in~\cite{wienobst2022applications} (Theorem~5).  

\paragraph*{Our contributions.}
In this work, we focus on the general problem of maximally
orienting a PDAG. The contributions of this paper are twofold: On the one hand, we give a
practical evaluation of methods for the consistent extension problem,
contributing two novel algorithms for this task, which are simple yet
effective and build upon the algorithm by~\cite{DorTarsi1992}. On the other hand, we aim to illustrate the
strengths of using consistent DAG extensions in the computation of the maximal
orientation of a PDAG. While this approach has been proposed theoretically
in~\citep{Chickering1995,WBL2021}, it
has only been used in special cases and has not found widespread
application in practice, evidenced by the fact that most
software packages such as \texttt{pcalg}~\citep{kalisch2012causal} and
\texttt{causaldag}~\citep{squires2018causaldag} rely on algorithms, which apply
the Meek rules directly. Utilizing consistent DAG extensions instead, yields a better
worst-case complexity and is superior in practice as we demonstrate in our
experiments.

\section{Preliminaries} \label{sec:prelim}
We consider partially directed graphs $G = (V_G,  A_G , E_G)$, in which the pairs
of vertices $x,y\in V$ are connected by 
directed edges $u \rightarrow v$ (also called arcs, given in $A_G$) or 
undirected edges $u-v$ (given in $E_G$).
We restrict
ourselves to graphs, where at most one edge exists between any pair $x,y \in V_G$
and if there is such an edge, we call $x$ and $y$ adjacent, denoted by $x \sim_G
y$. The arcs $u\rightarrow v$ and $u \leftarrow
v$ have different \emph{orientations} and \emph{orienting} an undirected edge $u
- v$ means replacing it with an arc.
Vertex $x$ is called a parent of $y$ if $x \rightarrow y \in G$, a child of $y$
if $x \leftarrow y \in G$ and a sibling of $y$ if $x - y \in G$. The sets of
parents, children and siblings of vertex $v$ are denoted by $\Pa_G(v)$, $\Ch_G(v)$,
$\Si_G(v)$. By $\Ne_G(v) = \Pa_G(v) \cup  \Ch_G(v) \cup \Si_G(v)$ we denote the set of
neighbors of $v$ and the {degree} of $v$ counts the number of its neighbors
$|\Ne_G(v)|$. Edges, which have $v$ as endpoint, are
called \emph{incident} to $v$. For a set $S \subseteq V_G$, the induced subgraph
$G[S]$ contains all edges from $G$ with both endpoints
in $S$.
We skip the subscript if $G$ is clear from the context.

A partially directed graph is called acyclic (or partially directed acyclic
graph, PDAG for short) if it does not contain a directed cycle, that is, a
sequence of distinct vertices $(c_1, c_2, \dots, c_k)$, with $k \geq 3$, and
edges $c_i \rightarrow c_{i+1} $ for $i \in \{1,\dots,k-1\}$, and $c_k
\rightarrow c_{1}$. A PDAG without any undirected edges ($E_G = \emptyset$) is
called a directed acyclic graph (DAG). There is a linear ordering (called
\emph{topological ordering}) of the vertices of every DAG
such that $u \rightarrow v$ if $u$ comes before $v$ in the ordering.

\begin{defi}
  A DAG $D$ is a \emph{consistent extension} of a PDAG $G$ if 
  \begin{enumerate}
    \item $D$ and $G$ have the same vertex set and $\Ne_G(v) = \Ne_D(v)$ for all
      vertices $v$,
    \item every directed edge in $G$ is also in $D$, i.\,e., $A_G \subseteq A_D$,
    \item for all edges $u - v$ in $E_G$, there is $u\rightarrow v$ in $A_D$ or
      $v\leftarrow u$ in $A_D$, and
    \item for all $u,v,w \in V$, the induced subgraph $u \rightarrow v
      \leftarrow w$ ($u \not\sim w$) is in $G$ iff it is in $D$. 
  \end{enumerate}
\end{defi}

The set of consistent extensions of a PDAG $G$ is denoted by $[G]$. We
remark that an induced subgraph of the form $u \rightarrow v \leftarrow w$ is
also called a v-structure. Hence, every $D\in[G]$ has the same v-structures as $G$.
If $[G] \neq \emptyset$, we call $G$ \emph{extendable}.
Some undirected edges in a PDAG $G$ might be oriented the same way in all of its
consistent extensions. The graph where these undirected edges are replaced by
the corresponding invariant arcs is called the \emph{maximal orientation} of
$G$, denoted by $\mathrm{MPDAG}(G)$, which stands for \emph{maximally-oriented
PDAG}. It is a well-known fact that $\mathrm{MPDAG}(G)$ can be computed by
repeatedly applying the four Meek rules~\citep{Meek1995}, which are shown in
Fig.~\ref{figure:meek:rules}, to $G$ until none applies anymore.
\begin{figure}
  \centering
  \begin{tikzpicture}[scale=0.9]
    \node (a) at (-1,0) {$a$};
    \node (b) at (-1,-1) {$b$};
    \node (c) at (0,-1) {$c$}; 
    \draw [arc](a) -- (b);
    \draw [-](b) -- (c);
    
    \node at (0.5,-0.5) {$\Rightarrow$};
    \node at (0.5, 0.25) {\textbf{R1}};
    
    \node (a) at (1,0) {$a$};
    \node (b) at (1,-1) {$b$};
    \node (c) at (2,-1) {$c$}; 
    \draw [arc](a) -- (b);
    \draw [arc](b) -- (c);

    \node (a) at (3,0) {$a$};
    \node (b) at (3,-1) {$b$};
    \node (c) at (4,-1) {$c$}; 
    \draw [arc](a) -- (b);
    \draw [arc](b) -- (c);
    \draw [-] (a) edge (c);
    
    \node at (4.5,-0.5) {$\Rightarrow$};
    \node at (4.5, 0.25) {\textbf{R2}};
    
    \node (a) at (5,0) {$a$};
    \node (b) at (5,-1) {$b$};
    \node (c) at (6,-1) {$c$};
    \draw [arc](a) -- (b);
    \draw [arc](b) -- (c);
    \draw [arc] (a) edge (c);
    
    \node (a) at (-1+8,-2+2) {$a$};
    \node (d) at (0+8,-2+2) {$d$};
    \node (b) at (-1+8,-3+2) {$b$};
    \node (c) at (0+8,-3+2) {$c$}; 
    \draw [-](a) -- (b);
    \draw [-](a) -- (d);
    \draw [-](a) -- (c);
    \draw [arc](b) -- (c);
    \draw [arc](d) -- (c);
    
    \node at (0.5+8,-2.5+2) {$\Rightarrow$};
    \node at (0.5+8, -1.75+2) {\textbf{R3}};
    
    \node (a) at (1+8,-2+2) {$a$};
    \node (d) at (2+8,-2+2) {$d$};
    \node (b) at (1+8,-3+2) {$b$};
    \node (c) at (2+8,-3+2) {$c$}; 
    \draw [-](a) -- (b);
    \draw [-](a) -- (d);
    \draw [arc](a) -- (c);
    \draw [arc](b) -- (c);
    \draw [arc](d) -- (c);      
    
    \node (a) at (3+8, -2+2) {$a$};
    \node (d) at (4+8, -2+2) {$d$};
    \node (b) at (3+8, -3+2) {$b$};
    \node (c) at (4+8, -3+2) {$c$};
    \draw [-] (a) -- (b);
    \draw [-] (a) -- (c);
    \draw [-] (a) -- (d);
    \draw [arc] (d) -- (c);
    \draw [arc] (c) -- (b);
    
    \node at (4.5+8, -2.5+2) {$\Rightarrow$};
    \node at (4.5+8, -1.75+2) {\textbf{R4}};
    
    \node (a) at (5+8, -2+2) {$a$};
    \node (d) at (6+8, -2+2) {$d$};
    \node (b) at (5+8, -3+2) {$b$};
    \node (c) at (6+8, -3+2) {$c$};
    \draw [arc] (a) -- (b);
    \draw [-] (a) -- (c);
    \draw [-] (a) -- (d);
    \draw [arc] (d) -- (c);
    \draw [arc] (c) -- (b);      
  \end{tikzpicture}
  \caption{The four Meek rules that are used to characterize MPDAGs~\citep{Meek1995}.}
  \label{figure:meek:rules}
\end{figure}
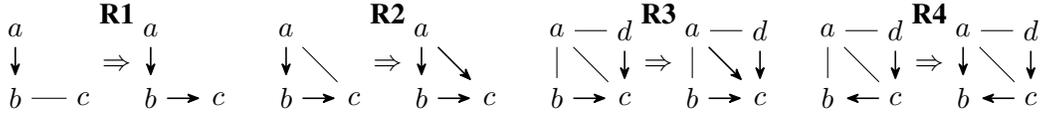
In this work, we consider the following two fundamental computational
problems for PDAGs:  

\vspace*{0.4cm}
\begin{minipage}{0.45\textwidth}
\begin{problem}{\textsc{ext}}
  \begin{computationalproblem}
    \input PDAG $G$.
    \output $D \in [G]$ or $\bot$ if $[G] = \emptyset$.
  \end{computationalproblem}
\end{problem}
\end{minipage}
\begin{minipage}{0.45\textwidth}
\begin{problem}{\textsc{max-orient}}
  \begin{computationalproblem}
    \input PDAG $G$.
    \output $\mathrm{MPDAG}(G)$.
  \end{computationalproblem}
\end{problem}
\end{minipage}
\vspace*{0.4cm}

For \textsc{max-orient}, we assume that $G$ is extendable, else $G$ has no
causal interpretation and the definition of $\mathrm{MPDAG}(G)$ is
without meaning. In this case, \textsc{ext} can be utilized for efficiently
solving \textsc{max-orient} as mentioned above (this is discussed in more
detail in Section~\ref{sec:applmeek}).
For $\textsc{ext}$, the notion of a \emph{potential-sink} is of central
importance. It asserts that all incident edges can be oriented towards a vertex
$v$ without introducing a new v-structure or changing the orientation of an arc.
Formally, vertex $v$ in $G$ is called a potential-sink if (i) there is no arc
$v \rightarrow x$ directed outward from $v$ (i.\,e., $\Ch(v) = \emptyset$) and
(ii) for every sibling $y \in \Si(v)$, vertex $y$ is adjacent to all other
neighbors of $v$. For example, vertex $e$ is the only potential-sink
in the graph $G$ shown in Fig.~\ref{fig:ce:mo:new}. 

The following fact states that a consistent extension of a graph $G$ can be
obtained by iteratively removing potential-sinks from $G$. This strategy is
known as the Dor-Tarsi algorithm.
\begin{fact}[\citealp{DorTarsi1992}] \label{fact:dt}
  Let $G_0$ be an extendable PDAG and let
  $G_i$, $1 \leq i \leq n$, be obtained by removing potential-sink $v_i$ and
  its incident edges from $G_{i-1}$. It holds that
  $(v_n, v_{n-1}, \dots, v_1)$ is a topological ordering of a
  consistent extension of $G_0$.  
\end{fact}
In particular, it is true that every $G_i$ has a potential-sink due to the
observation that removing a vertex and its incident edges from an extendable
graph, again yields an extendable graph (extendability is closed under taking
subgraphs).

%\begin{figure}
%  \centering
%  \begin{tikzpicture}[yscale=0.75]
%    \node (desc) at (0,.9) {$G$};
%    \node (a) at (0,0) {$a$};
%    \node (b) at (1,1) {$b$};
%    \node (c) at (1,-1) {$c$};
%    \node (d) at (2,0) {$d$};
%    \node (e) at (3,0) {$e$};
%
%    \draw [-] (a) -- (b);
%    \draw [-] (a) -- (c);
%    \draw [arc] (b) -- (d);
%    \draw [arc] (c) -- (d);
%    \draw [-] (d) -- (e);
%    \draw [arc] (b) -- (c);
%  \end{tikzpicture}
%  \hspace*{7mm}
%  \begin{tikzpicture}[yscale=0.75]
%    \node (desc) at (0,.9) {$D$};
%    \node (a) at (0,0) {$a$};
%    \node (b) at (1,1) {$b$};
%    \node (c) at (1,-1) {$c$};
%    \node (d) at (2,0) {$d$};
%    \node (e) at (3,0) {$e$};
%
%    \draw [arc] (b) -- (a);
%    \draw [arc] (c) -- (a);
%    \draw [arc] (b) -- (d);
%    \draw [arc] (c) -- (d);
%    \draw [arc] (d) -- (e);
%    \draw [arc] (b) -- (c);
%  \end{tikzpicture}
%   \hspace*{7mm}
%   \begin{tikzpicture}[yscale=0.75]
%   \node (desc) at (0,.9) {$M$};
%    \node (a) at (0,0) {$a$};
%    \node (b) at (1,1) {$b$};
%    \node (c) at (1,-1) {$c$};
%    \node (d) at (2,0) {$d$};
%    \node (e) at (3,0) {$e$};
%
%    \draw [-] (a) -- (b);
%    \draw [-] (a) -- (c);
%    \draw [arc] (b) -- (d);
%    \draw [arc] (c) -- (d);
%    \draw [arc] (d) -- (e);
%    \draw [arc] (b) -- (c);
%  \end{tikzpicture}
%
% \caption{A PDAG $G$ with a consistent extension $D$ and its maximal
% orientation $M$. The vertices $a$ and $e$ are the only potential-sinks in $G$.}
%  \label{fig:ce:mo}
%\end{figure}

These notions are exemplified in Fig.~\ref{fig:ce:mo:new}. The figure shows that the
consistent extension $D$ can be obtained by iteratively identifying a
potential-sink (at the start, $e$ is the only potential-sink), orienting its
incident undirected edges towards it,
and continuing on the induced subgraph over the remaining vertices.
In this particular case, $d$ would be the potential-sink found afterward,
followed by $b$ and $c$ (in arbitrary order), and finally $a$. 

The graph $M$ is the maximal orientation of $G$.
It contains the arc $d \rightarrow e$ as $d \leftarrow e$ cannot occur in
any consistent extension of $G$ (it creates a new v-structure that is not in
$G$, see Meek rule R1) and $a \rightarrow d$ follows from Meek rule R3,
whereas $a - b$ and $a - c$ are oriented in both directions in different consistent
extensions of $D$ and are consequently undirected in $M$.

\section{Two New Simple Algorithms for Extendability} \label{sec:newext}
As discussed above, algorithms for extending PDAGs play a key role in the
maximal orientation task. While, from a theoretical point of view, previous results suggest that it is likely not possible to further improve the
asymptotic run time of extendability algorithms, as \citet{WBL2021} gave a
conditional $O(n^3)$ lower bound for combinatorial algorithms, from a
practical perspective, the current algorithms are either extremely simple, as
Dor-Tarsi's approach, or very sophisticated with considerable practical overhead,
as the methods proposed by~\citet{WBL2021}.
In this work, we are searching for a middle ground and give
two novel approaches to extend PDAGs, with the main focus lying on simplicity
and effectiveness. The first approach is a direct modification of the Dor-Tarsi
algorithm, which greatly improves its empirical performance. The second one
gives a practical $O(n^3)$ algorithm, which is, however, conceptually simpler than the one
presented by~\citet{WBL2021} and has significantly less overhead.

\subsection{Dor-Tarsi with Degree-Heuristic}
The Dor-Tarsi algorithm has an asymptotic run time of $O(n^4)$ and is
computationally simplistic. It iteratively identifies a potential-sink and then
removes it and its incident edges from the graph. Hence, the whole algorithm consists of $n$
iterations, each amounting to the task of finding a potential-sink. Naively, a
potential-sink can be found in time $O(n^3)$ by going through all vertices and
checking for each vertex whether it satisfies the potential-sink property. 
Implemented this way, each vertex is repeatedly tested for being a
potential-sink throughout the algorithm (until it eventually becomes a
potential-sink and is subsequently removed from the graph). While this
brute-force method appears to be severely worse than more clever approaches,
which store and update relevant information such as a list of all
potential-sinks in sophisticated data structures, there are two notable advantages of
this approach: (i) whenever a potential-sink is found, the loop over all
vertices can be exited immediately, and (ii) the check for potential-sinkness
can be cancelled once a single missing edge violates the potential-sink
property. 

Hence, the first observation of this work, discussed in more detail throughout
the experiments in Section~\ref{sec:eval}, is that in many cases, a naive
implementation of the Dor-Tarsi algorithm performs quite well in practice (in
particular in cases where many potential-sinks cause frequent early exits of
the for-loop and also for graphs in which the test for potential-sinkness
usually fails early), being competitive with more subtle approaches. 
\begin{algorithm2e}[t]
  \caption{A heuristic implementation of Dor-Tarsi iterating over the vertices in
    order of increasing degree.}
  \label{alg:dth}
  \DontPrintSemicolon
  \LinesNumbered
  \SetKwInOut{Input}{input}\SetKwInOut{Output}{output}
  \SetKw{Break}{break}
  \SetKwFor{Rep}{repeat}{}{end}
  \Input{A PDAG $G = (V,A,E)$.}
  \Output{$D \in [G]$ or $\bot$ if $[G] = \emptyset$.}

  $D := (V, \emptyset)$ \; 
  \Rep{$n$ times}{
    \For{$v \in V$ in increasing degree in $G$ \label{line:dth:innerfor}}{ 
      \If{$v$ is a potential-sink}{
        Remove $v$ and its incident edges from $G$. \;
        Add arcs $\{(u,v) \mid u \in \Si_G(v) \cup \Pa_G(v)\}$ to $D$. \;
        \Break \;
      }
    }
    \If{no potential-sink has been found}{
      \Return $\bot$ \;
    }
  }

  \Return $D$ \;
\end{algorithm2e}
Building on the surprising effectiveness of the Dor-Tarsi method, a heuristic
refinement is given in Algorithm~\ref{alg:dth}. The idea
is to go through the vertices in order of increasing degree. Using this
heuristic leads to fewer cost in case the loop is exited early as the iteration
in line~\ref{line:dth:innerfor} always starts with the ``cheapest" vertices,
i.\,e., those with minimal degree (checking for potential-sinkness results in
worst-case costs quadratic in the number of neighbors $O(|\Ne(v)|^2)$). Fig.~\ref{fig:dh:ex}
exemplifies this advantage.
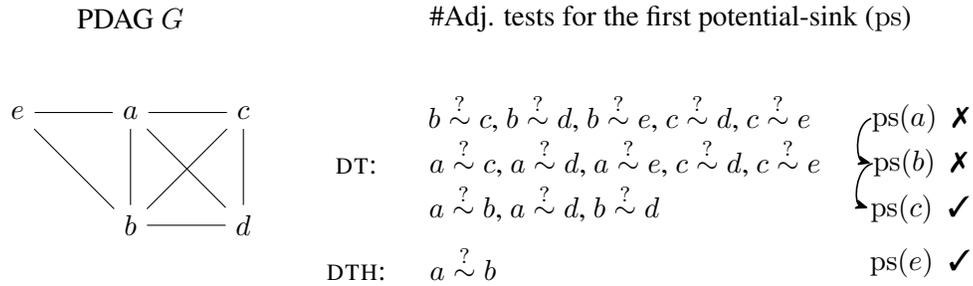
\begin{figure}
  \centering
  \begin{tikzpicture}
    \node (a) at (0,0) {$a$};
    \node (b) at (0,-1.5) {$b$};
    \node (c) at (1.5,0) {$c$};
    \node (d) at (1.5,-1.5) {$d$};
    \node (e) at (-1.5,0) {$e$};

    \draw [-] (a) -- (b);
    \draw [-] (a) -- (c);
    \draw [-] (a) -- (d);
    \draw [-] (a) -- (e);
    \draw [-] (b) -- (c);
    \draw [-] (b) -- (d);
    \draw [-] (b) -- (e);
    \draw [-] (c) -- (d);

    \node (lg) at (0, 1.25) {
      PDAG $G$
    };

    \node (l) at (7.18, 1.25) {
      \#Adj. tests for the first potential-sink ($\mathrm{ps}$)
    };

    \node (dth) at (3,-0.7125) {
      \textsc{dt}:
    };
    \node (ca) at (6.5,0) {
      $b \isadj c$, $b \isadj d$, $b \isadj e$, $c \isadj d$, $c \isadj e$ 
    };
    \node (cb) at (6.575,-.6) {
      $a \isadj c$, $a \isadj d$, $a \isadj e$, $c \isadj d$, $c \isadj e$ 
    };
    \node (cc) at (5.5,-1.2) {
      $a \isadj b$, $a \isadj d$, $b \isadj d$
    };
    \node (dth) at (3,-2.125) {
      \textsc{dth}:
    };
    \node (ce) at (4.42, -2) {
      $a \isadj b$
    };

    \node[inner sep=0] (psa) at (10.5,-.08) {
      $\mathrm{ps}(a) \;$ \xmark
    };

    \node[inner sep=0] (psb) at (10.5, -0.68) {
      $\mathrm{ps}(b) \;$ \xmark
    };

    \node[inner sep=0] (psc) at (10.5, -1.28) {
      $\mathrm{ps}(c) \;$ \cmark
    };

    \node[inner sep=0] (pse) at (10.5, -2) {
      $\mathrm{ps}(e) \;$ \cmark
    };

    \draw[arc] (psa.west) to[bend right=75] (psb.west);
    \draw[arc] (psb.west) to[bend right=75] (psc.west);
\end{tikzpicture}

  \caption{
    A PDAG and the adjacency tests required by Dor-Tarsi (\textsc{dt}) and the
    heuristic adaption (\textsc{dth}) for finding the first potential-sink. We
    assume that \textsc{dt} iterates over the vertices in alphabetical order.
    \textsc{dt} checks vertices  $a$ and $b$ first, which are no potential-sinks.
    \textsc{dth} starts with the lowest-degree vertex $e$. This reduces the cost of
    testing potential-sinkness and increases the chances of finding a potential-sink early. 
  }
  \label{fig:dh:ex}
\end{figure}

It demonstrates that checking potential-sinkness
for the low-degree vertices
first has, on the one hand, lower cost per vertex and, on the other hand, is often more likely to
succeed early (for example, vertices with degree one and no outgoing edges are always
potential-sinks). We note that some overhead is induced by this approach because
the degree of the vertices has to be continuously updated.
However, this cost is mostly tolerable as our experiments presented in Section~\ref{sec:eval}
confirm and the heuristic yields a simple and practical improvement over the
standard Dor-Tarsi algorithm. Clearly, a heuristic is not always optimal and
worst-case examples can be constructed where the algorithm's run time reaches its
upper asymptotic bound:

\begin{fact}
  There are instances on which the Dor-Tarsi heuristic yields a run time of $\Omega(n^4)$.
\end{fact}

\begin{proof}
  Consider graph $G = (V,A=\emptyset,E)$ with $V = \{v, w\} \cup C_v \cup C_w \cup C_{vw}$, 
  where $C_{vw}$ consists of $k$ vertices and $C_v$ and $C_w$ consist of $2k$
  vertices each. The edge set $E$ is constructed as
%  \[
%    E = \{\{x,y\} \mid x,y \in C_L \text{ with } L \in \{v, w, vw\}\} \cup
%    (\{v,w\} \times C_{vw}) \cup (\{v\} \times C_v) \cup (\{w\} \times C_w).
%  \]
  \[
    E = \{\{x,y\} \mid x,y \in C_{L} \text{ for } L = v, w, vw\} \cup
     \{ \{v,x\}\mid x\in C_{v} \cup C_{vw}\} \cup
    \{ \{w,x\}\mid x\in C_{w} \cup C_{vw}\} .
  \]
  In words, the cliques $C_v$, $C_w$ are fully connected to $v$ and $w$,
  respectively, and $C_{vw}$ is fully connected to both $v$ and $w$.
  Moreover, $v$ and $w$ are non-adjacent.
  The graph is illustrated in Fig.~\ref{fig:ex:counter}.

  Dor-Tarsi in combination with the degree-heuristic always iterates over the
  vertices in $C_{vw}$ with initial degree $k+1$ first. None of these vertices
  are, however, potential-sinks. Consequently, $\Omega(k)$ vertices are checked 
  before a potential-sink is found. Moreover, testing for potential-sinkness
  only fails for one pair of neighbors, that is, for $v$ and $w$, causing the
  loop over the pairs of neighbors to require $\Omega(k^2)$ iterations before exiting
  with high probability. Note that for $n=5k+2$, $k$ is in $\Omega(n)$. Overall, the steps
  undertaken by the heuristic yield a run time of $\Omega(n^4)$.
\end{proof}

\begin{figure}
  \centering
  \begin{tikzpicture}[yscale=0.8]
    \node (v) at (0,-1) {$v$};
    \node (w) at (6,-1) {$w$};

    \node (l1) at (0,1.7) {$2k$ vertices};
    \draw[decorate, ultra thick, decoration= {calligraphic brace, amplitude=7pt}] (-2.2,1) -- (2.2,1);

    \node[draw,circle,inner sep=1.2pt,fill] (1) at (-2,0) {};
    \node[draw,circle,inner sep=1.2pt,fill] (2) at (-1,0) {};
    \node[] (3) at (0,0) {\dots};
    \node[draw,circle,inner sep=1.2pt,fill] (4) at (1,0) {};
    \node[draw,circle,inner sep=1.2pt,fill] (5) at (2,0) {};

    \foreach \x in {1,2,3,4,5} {
      \foreach \y in {1,2,3,4,5} {
        \pgfmathsetmacro\xp{\x+1}
        \ifdim \xp pt<\y pt
        \draw [-] (\x) to[bend left] (\y);
        \fi
      }
    }

    \draw[-] (1) -- (2);
    \draw[-] (2) -- (3);
    \draw[-] (3) -- (4);
    \draw[-] (4) -- (5);

    \foreach \x in {1,2,3,4,5} {
      \draw [-] (v) -- (\x);
    }
    
    \node (l2) at (6,1.7) {$2k$ vertices};
    \draw[decorate, ultra thick, decoration= {calligraphic brace,
      amplitude=7pt}] (3.8,1) -- (8.2,1);
    \node[draw,circle,inner sep=1.2pt,fill] (1) at (4,0) {};
    \node[draw,circle,inner sep=1.2pt,fill] (2) at (5,0) {};
    \node[] (3) at (6,0) {\dots};
    \node[draw,circle,inner sep=1.2pt,fill] (4) at (7,0) {};
    \node[draw,circle,inner sep=1.2pt,fill] (5) at (8,0) {};
   
    \foreach \x in {1,2,3,4,5} {
      \foreach \y in {1,2,3,4,5} {
        \pgfmathsetmacro\xp{\x+1}
        \ifdim \xp pt<\y pt
        \draw [-] (\x) to[bend left] (\y);
        \fi
      }
    }

    \draw[-] (1) -- (2);
    \draw[-] (2) -- (3);
    \draw[-] (3) -- (4);
    \draw[-] (4) -- (5);
    
    \foreach \x in {1,2,3,4,5} {
      \draw [-] (w) -- (\x);
    }

    \node[draw,circle,inner sep=1.2pt,fill] (1) at (1,-2) {};
    \node[draw,circle,inner sep=1.2pt,fill] (2) at (2,-2) {};
    \node[] (3) at (3,-2) {\dots};
    \node[draw,circle,inner sep=1.2pt,fill] (4) at (4,-2) {};
    \node[draw,circle,inner sep=1.2pt,fill] (5) at (5,-2) {};

    \foreach \x in {1,2,3,4,5} {
      \foreach \y in {1,2,3,4,5} {
        \pgfmathsetmacro\xp{\x+1}
        \ifdim \xp pt<\y pt
        \draw [-] (\x) to[bend right] (\y);
        \fi
      }
    }

    \draw[-] (1) -- (2);
    \draw[-] (2) -- (3);
    \draw[-] (3) -- (4);
    \draw[-] (4) -- (5);

    \foreach \x in {1,2,3,4,5} {
      \draw [-] (v) -- (\x);
      \draw [-] (w) -- (\x);
    }
    \node (l2) at (3,-3.7) {$k$ vertices};
    \draw[decorate, ultra thick, decoration= {calligraphic brace, mirror,
      amplitude=7pt}] (.8,-3) -- (5.2,-3);
  \end{tikzpicture}

  \caption{An example instance where the heuristic exhibits run time $\Omega(n^4)$.}
  \label{fig:ex:counter}
\end{figure}
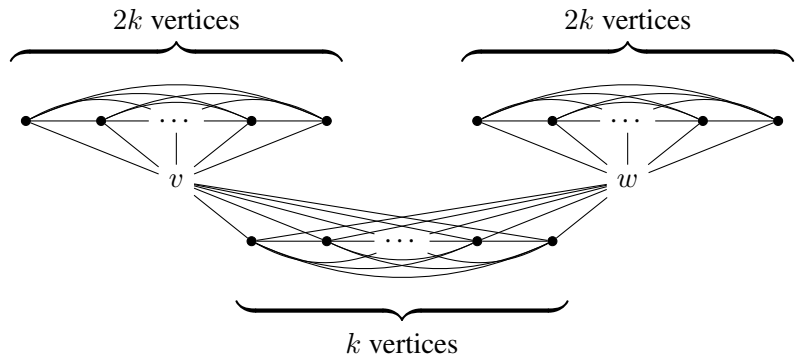

\subsection{Dor-Tarsi with Improved Worst-Case Complexity} \label{sec:dtic}
In the previous section, a direct and effective heuristic improvement of the
Dor-Tarsi algorithm is presented. In this section, we aim to give another
simple adaptation of Dor-Tarsi, but this time with an $O(n^3)$ worst-case run
time, thus matching the lower-bounds of the algorithm given in~\citep{WBL2021}.

If we take a look at the counterexample from Fig.~\ref{fig:ex:counter} again,
we observe that Algorithm~\ref{alg:dth} repeatedly iterates over the same
neighborhoods when searching for a potential-sink, i.\,e., there is no
information stored between iterations. It is clearly desirable to avoid this
repeated computational effort. Previously, improvements were made by
constructing elaborate data structures, which maintain the set of
potential-sinks throughout the course of the algorithm~\citep{WBL2021}. Such
data structures, however, induce a significant practical overhead and, in
particular, demand an expensive initialization step at the start of the
algorithm, even before searching for the first potential-sink.

To address these issues, our proposed approach iterates over the neighbors of
every vertex $v$ exactly once, even though potential-sinkness of $v$ might be
checked multiple times during the course of the algorithm. After the neighbors
of a vertex $v$ have been iterated over, all tuples of neighbors violating the
potential-sink property of $v$ are stored in the set $B[v]$, which is afterward
continuously updated every time vertices are removed from the graph. The sets
$B[v]$ are computed lazily, that is, only at the time the potential-sinkness of
$v$ is checked for the first time and thus there is no pre-computation
involved. In case the set of violating neighbors $B[v]$ is empty or becomes
empty by removing other vertices, $v$ satisfies the potential-sink property,
which is due to the observation that throughout the course of the algorithm no
\emph{new} potential-sinkness violations are incurred.

\begin{algorithm2e}[t]
  \caption{An adaptation of Dor-Tarsi with worst-case complexity $O(n^3)$.}
  \label{alg:dtch}
  \DontPrintSemicolon
  \LinesNumbered
  \SetKwInOut{Input}{input}
  \SetKwInOut{Output}{output}
  \SetKw{Break}{break}
  \Input{A PDAG $G = (V,A,E)$.}
  \Output{$D \in [G]$ or $\bot$ if $[G] = \emptyset$.}

  $D := (V, \emptyset)$ \;
  $B := \text{array of } n \text{ initially empty sets}$\;
  $C := \text{bitvector of length } n \text{ initialized with }  \textit{false}$\;
  \While{there are vertices left in $G$}{ \label{line:dtch:outerloop}
    \ForEach{$v \in V$ in increasing degree in $G$}{ \label{line:dtch:vloop}
      \If{$C[v] =  \textit{false}$ and $\Ch_G(v) = \emptyset$}{ \label{line:dtch:begininner}
        $B[v] := B[v] \cup \{(u, u') \mid u \in \Si_G(v), u' \in \Si_G(v) \cup \Pa_G(v) \land u \neq u' \land u \not\sim_G u'\}$ \; \label{line:dtch:nbrupdate}
        $C[v] :=  \textit{true}$ \;
      }
      \If{$C[v] =  \textit{true}$ and $B[v] = \emptyset$}{ \label{line:dtch:beginrem}
        Remove $v$ and its incident edges from $G$. \;
        Add arcs $\{(u,v) \mid u \in \Si_G(v) \cup \Pa_G(v)\}$ to $D$. \;
       % Remove $v$ from all sets in $B$ containing a tuple where $v$ is included. \;
        Remove from all sets in $B$ tuples  including $v$. \;
        \Break
      } \label{line:dtch:endrem}
    }
    \If{no potential-sink has been found}{
      \Return $\bot$ \;
    }
  }
  \Return $D$ \;
\end{algorithm2e}
The whole approach is described in detail in Algorithm~\ref{alg:dtch}. We
combine it with the heuristic from the previous section, which is generally
favorable to use. Algorithm~\ref{alg:dtch} combines the advantages of both
\textsc{dt} and the algorithm from~\cite{WBL2021} by, one the one hand,
storing information from previous iterations to ensure an $O(n^3)$ worst-case
run time while, on the other hand, keeping the overhead to a minimum and,
in particular, not relying on an initialization step.

\begin{thm}
  Algorithm~\ref{alg:dtch} is implementable with expected worst-case time complexity $O(n^3)$.
\end{thm}

\begin{proof}
  We assume that edge deletion and adjacency tests are supported in constant
  time and iterating over the neighbors of $v$ has cost $O(|\Ne_G(v)|)$. Using hash
  tables to store the neighbors of a vertex, these run times are reached
  \emph{in expectation}. The graph representation is discussed in more detail
  in the appendix (Section~\ref{appendix:graph-repr}).

  The computation of neighbors violating the potential-sink property in
  line~\ref{line:dtch:nbrupdate} can clearly be implemented in $O(n^2)$ by
  looping over all pairs of neighbors of $v$ and checking adjacency.
  Due to the flag $C[v]$, the neighborhood for each vertex $v$ is visited
  exactly once and thus line~\ref{line:dtch:nbrupdate} is executed $O(n)$
  times, yielding a total run time of $O(n^3)$ for searching potential-sinks.
  The removal of a potential-sink in lines~\ref{line:dtch:beginrem} to
  \ref{line:dtch:endrem} runs in time $O(n^2)$ as $v$ has at most $n$ neighbors
  for which edges are removed from $G$ and added to $D$ and there are at most
  $O(n^2)$ tuples in the sets in $B$ where $v$ is included.
  By storing all tuples where $v$ is included separately, we can iterate over
  them and remove each tuple from the corresponding set in $O(1)$.
  As every vertex is removed at most once from the graph, the removal is
  repeated $O(n)$ times, yielding a run time of $O(n^3)$ for the removal of all
  vertices. Consequently, the run time of the whole algorithm is bounded by
  $O(n^3)$.
\end{proof}

The correctness of the algorithm follows immediately from the correctness of
the Dor-Tarsi algorithm (Fact~\ref{fact:dt}) as our proposed modification only avoids repeated
visits to neighbors but does not skip any neighbor check. As a final remark, we
note that Algorithm~\ref{alg:dtch} is also implementable using a combination of
linked lists and an adjacency matrix to
represent the neighbors of a vertex $v$, as proposed by~\citep{WBL2021}.
Using such a representation, instead
of collecting all neighbors of $v$ violating the potential-sink property in
line~\ref{line:dtch:nbrupdate}, one could stop as soon as the first pair of
neighbors violating the property is found and store a pointer to the violating
neighbor, allowing to start the iteration over the neighbors of $v$ the next
time at the lastly visited neighbor\footnote{Such a strategy is not possible
for hashed data structures due to potential rehashing.}. However, using linked
lists instead of hash sets to represent adjacencies has no impact on the
worst-case complexity and is not significantly faster than the hashed data
structure. More details about the comparison between these two representations
can be found in Appendix~\ref{appendix:graph-repr}.

\section{Evaluation of Extension Algorithms} \label{sec:eval}
In this section, we conduct an experimental evaluation of the algorithms
discussed in the previous sections to demonstrate their practical effectiveness.
We compare the Dor-Tarsi algorithm (\textsc{dt}), its heuristic
refinement (\textsc{dth}), its adaptation with improved worst-case complexity 
(\textsc{dtic}), and the algorithm given by~\cite{WBL2021} (\textsc{wbl}).
All algorithms are implemented in Julia~\citep{Bezanson2017}.
%and the graphs are
%represented by storing the neighbors of each vertex in hash tables.
We present the results for random PDAGs, which are generated by
(i) creating a random DAG $D$, (ii) replacing all directed edges not participating
in a v-structure by an undirected edge, and (iii) orienting between 
two and five (randomly chosen) undirected edges according to $D$,
resulting in an extendable PDAG.
The initial DAG $D$ is generated by creating a random undirected graph and
afterward using a random permutation of its vertices as a topological ordering
according to which the edges are then oriented. We vary the number of vertices
in the input graphs by setting $n = 128, 256, \dots, 8192$ and the number of
edges is set to $m = 3 \cdot n, 5 \cdot n, \log_2(n) \cdot n, \sqrt n \cdot n$.
For each choice of parameters, we generate ten instances and then run every
algorithm ten times on each instance. The run times reported in this section
and in upcoming sections are averages over all runs on those ten instances.
Further experimental results for scale-free PDAGs and chordal graphs are
given in Appendix~\ref{appendix:further-eval-ext}.

\begin{figure}
  \centering
  \resizebox{\textwidth}{!}{\input{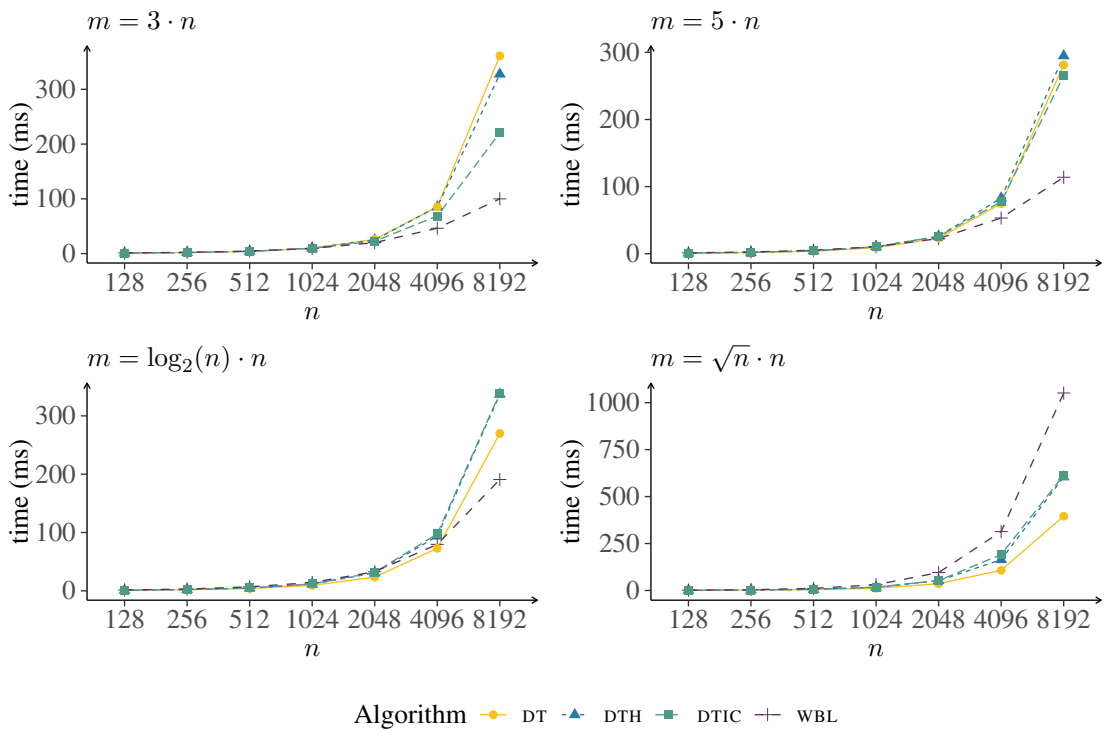}}
  \caption{Run times of the algorithms \textsc{dt}, \textsc{dth},
    \textsc{dtic}, and \textsc{wbl} on randomly generated PDAGs of $n$ vertices and $m$ edges,  
    with $m = 3 \cdot n$  (top left), $m = 5 \cdot n$  (top right),
    $m = \log_2(n) \cdot n$  (bottom left), and $m = \sqrt n \cdot n$
    (bottom right).}
  \label{fig:results-ext-pdag-er}
\end{figure}

The results are shown in Fig.~\ref{fig:results-ext-pdag-er}. In the top left
plot, the run times on sparse graphs with $m = 3 \cdot n$ edges are given.  %illustrated.
For small graphs, there are no visible differences among the run times of the
algorithms. However, with an increasing number of vertices, it becomes clear
that \textsc{wbl} performs best, which is quite expected, as it was shown to
have linear-time $O(m)$ for sparse graphs (formalized as constant-degeneracy
graphs in~\citep{WBL2021}). 
The run time of \textsc{dt} increases rapidly as the number of vertices grows,
while both \textsc{dth} and \textsc{dtic} exhibit a slightly smaller increase.
Although \textsc{dtic} does not outperform \textsc{wbl}, it shows a significant
improvement compared to \textsc{dt}.

In both the top right plot and the bottom left plot, \textsc{wbl} yields still
the best performance among the algorithms.
For denser graphs, with $m = \sqrt n \cdot n$ (bottom right),
\textsc{wbl} yields a weaker performance than the other algorithms. This can be
explained by the fact that \textsc{wbl} relies on an initialization step, which
requires an iteration of every pair of neighbors for each vertex in the graph
and thereby induces costs of $O(n^3)$ for dense graphs. Those costs emerge even
before searching for the first potential-sink and they are particularly high as
the initialization loops cannot be exited early. All of the three other
approaches do allow for early exits, which turns out to be a
non-negligible advantage because the graph is getting smaller and smaller
during the course of the algorithm, meaning that the high initialization effort
of \textsc{wbl} dominates the run time. Particularly for denser
graphs, \textsc{dt} is
even slightly faster than \textsc{dth} and \textsc{dtic}, which indicates that
both \textsc{dth} and \textsc{dtic} induce some overhead by continuously
maintaining the vertices in sorted order by their degree.

Overall, \textsc{dtic} provides a stable performance across different
graph sizes and graph densities, verifying that it combines the advantages of
\textsc{dt} and \textsc{wbl}. \textsc{dt} is often a solid choice in practice
but its performance depends heavily on the order of checking vertices for
potential-sinkness and thus combining its advantages with a guaranteed
worst-case run time of $O(n^3)$ yields a promising algorithm for practical
applications. 

\section{Application to Maximal Orientations} \label{sec:applmeek}
Building on the results in the previous sections, we are able to demonstrate an
immediate application of extension algorithms. As shown by~\citet{Chickering1995}
and~\citet{WBL2021}, it is possible to utilize extension
algorithms when computing the maximal orientation of a PDAG $G$. In practice, it is common
to implement the step from PDAG to MPDAG the ``direct way", that is,
by repeatedly applying the Meek rules (shown in
Fig.~\ref{figure:meek:rules}) in a while-loop. We argue that using
consistent extensions not only gives desirable worst-case guarantees,
but also performs remarkably well in practice.

\subsection{How to Use Extendability for the Computation of Maximal Orientations}
Applying the Meek rules repeatedly to a PDAG $G$ yields its maximal
orientation. A direct implementation, which loops over the given graph repeatedly is
computationally expensive, leading to worst-case run times such as $O(n^4
\cdot m)$\footnote{$O(m)$ edges might be oriented successively and
naively checking the applicability of the Meek rules results in costs of $O(n^4)$.}.
%
%\begin{figure}
%  \begin{tikzpicture}
%    \node (a) at (0,0) {$a$};
%    \node (b) at (1,0) {$b$};
%    \node (c) at (2,0) {$c$};
%    \node (d) at (3,0) {$d$};
%    \node (e) at (4,0) {$e$};
%
%    \draw [-] (a) -- (b);
%    \draw [-] (a) to[bend right] (c);
%    \draw [arc] (b) -- (c);
%    \draw [arc] (b) to[bend left] (e);
%    \draw [-] (c) -- (d);
%    \draw [-] (c) to[bend left] (e);
%    \draw [arc] (d) -- (e);
%  \end{tikzpicture}
%  \begin{tikzpicture}
%    \node (a) at (0,0) {$a$};
%    \node (b) at (1,0) {$b$};
%    \node (c) at (2,0) {$c$};
%    \node (d) at (3,0) {$d$};
%    \node (e) at (4,0) {$e$};
%
%    \draw [arc] (a) -- (b);
%    \draw [arc] (a) to[bend right] (c);
%    \draw [arc] (b) -- (c);
%    \draw [arc] (b) to[bend left] (e);
%    \draw [arc] (c) -- (d);
%    \draw [arc] (c) to[bend left] (e);
%    \draw [arc] (d) -- (e);
%  \end{tikzpicture}
%  \begin{tikzpicture}
%    \node (a) at (0,0) {$a$};
%    \node (b) at (1,0) {$b$};
%    \node (c) at (2,0) {$c$};
%    \node (d) at (3,0) {$d$};
%    \node (e) at (4,0) {$e$};
%
%    \draw [-] (a) -- (b);
%    \draw [-] (a) to[bend right] (c);
%    \draw [arc] (b) -- (c);
%    \draw [arc] (b) to[bend left] (e);
%    \draw [arc] (c) -- (d);
%    \draw [arc] (c) to[bend left] (e);
%    \draw [arc] (d) -- (e);
%  \end{tikzpicture}
%
%  \caption{A PDAG (left), a consistent extension (middle) of it, and its
%  maximal orientation (right).}
%  \label{fig:mo:example}
%\end{figure}
%
Clearly, it is preferable to traverse the graph only a single time, while
deciding which edges should be oriented. The crucial observation is that this
can be achieved by utilizing a topological ordering of a consistent extension
of $G$.
%
%Fig.~\ref{fig:mo:example} illustrates this idea. When traversing the vertices
%this way, the Meek rules can be propagated from left-to-right. That is, it is
%checked for vertex $d$ (under the inductive assumption that all edges having
%their right endpoint left to $d$ have already been handled) whether one of the
%Meek rules applies to its incident edges having their other endpoint located
%left from $d$. In this example, Meek rule 1 applies to the edge $c - d$. The
%appeal of this approach is that, as we have seen before, a consistent extension
%can often be found extremely fast using the algorithms presented in this paper. 
%
%
%

More precisely, when computing the maximal orientation of a PDAG, one can
distinguish between two situations: (i) the result of applying the
Meek rules yields a CPDAG, such as in the final phase of the PC algorithm, and
(ii) in case of additional background knowledge, the resulting graph is not
necessarily a CPDAG, but a maximally oriented PDAG (i.\,e., an MPDAG).

In case (i), the CPDAG can be obtained without applying the Meek rules
by extending the PDAG into a DAG and afterward computing the corresponding
CPDAG directly from the DAG. Indeed, for this second step from DAG to CPDAG, \citet{Chickering1995}
gave a linear-time algorithm\footnote{The algorithm proceeds by checking for
each edge whether it should be undirected. Although not utilizing the Meek
rules, the algorithm makes heavy use of the topological ordering provided by the
consistent extension.}. 
%We show in our experiments that, on the one hand, the extension
%step makes up the largest proportion of the run time (as finding the CPDAG from
%the DAG is extremely fast) and, on the other hand, that using this strategy is
%much faster than applying the Meek rules directly (which is due to the
%effective extension algorithms).
%
In case (ii), starting with the CPDAG obtained as in (i), that is, by performing
Chickering's DAG-to-CPDAG algorithm on a consistent extension $D$,
further orienting all edges which are directed in $G$ (i.\,e., the background knowledge
edges), and afterward applying the Meek rules in a single iteration over the vertices
of the graph in order of a topological ordering of $D$ yields the maximal
orientation~\citep{WBL2021}. While this procedure appears quite involved, we
demonstrate that the extension step is actually the
most expensive part and the subsequent steps are comparably cheap.
Moreover, we show that this strategy is significantly faster than applying the
Meek rules directly. From a theoretical point of view, these approaches
guarantee an $O(n^3)$ worst-case run time, which is a significant improvement as
well.

%\begin{figure}
%  \centering
%  \begin{tikzpicture}
%    \node (pdag)  at (0,0) {PDAG};
%    \node (dag)   at (1.5,-2) {DAG};
%    \node (cpdag) at (3,0) {CPDAG};
%    \node (mpdag) at (5.5,0) {MPDAG};
%    \draw[arc] (pdag) to node[midway,below left](){Ext.} (dag);
%    \draw[arc] (pdag) to node[midway,above](){MR} (cpdag);
%    \draw[arc] (dag) to node[midway,right](){$\dagger$} (cpdag);
%    \draw[arc] (3.75,-1) to node[midway,black,below right](){MR (topological order)} (mpdag);
%    \draw[arc, bend left] (pdag) to node[midway,above](){MR} (mpdag);
%    \draw[semithick] (-1,.75) rectangle (3.75,-2.5);
%  \end{tikzpicture}
%
%  \caption{The different ways of obtaining the maximal orientation of a PDAG in form of
%    a CPDAG or MPDAG (MR: Meek rules, $\dagger$: \citealp{Chickering1995}).}
%  \label{fig:mo:routes}
%\end{figure}
%
%An overview of the different ways towards the maximal orientation is
%illustrated in Fig.~\ref{fig:mo:routes}.

\begin{figure}
  \centering
  \resizebox{\textwidth}{!}{\input{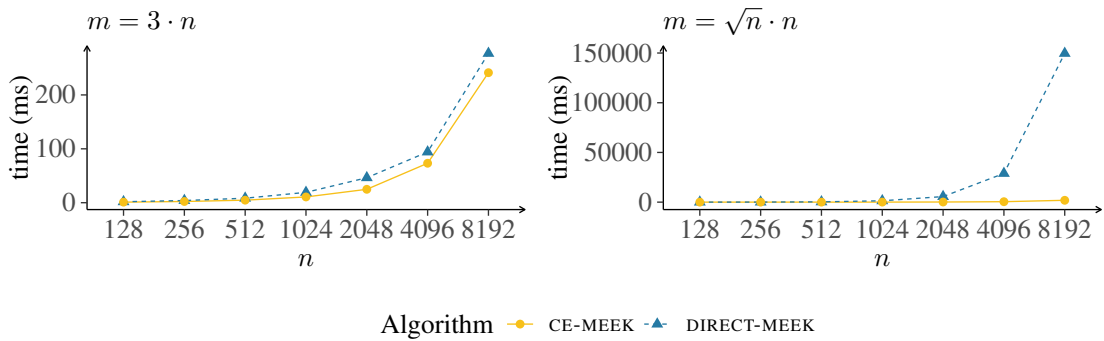}}
  \caption{Run times of the algorithms \textsc{direct-meek} and \textsc{ce-meek}
    on randomly generated PDAGs with $m = 3 \cdot n$ edges (left) and
    $m = \sqrt n \cdot n$ edges (right).}
  \label{fig:results-mo-pdag-er}
\end{figure}

\subsection{Experimental Comparison with the Direct Application of the Meek Rules} \label{sec:appl-meek-eval}
In this section, we compare the two approaches to maximally orient a given PDAG
in an experimental evaluation. More specifically, we compare the direct
application of Meek's rules (\textsc{direct-meek}) to the above mentioned approach originally
introduced by \citet{Chickering1995} and generalized to arbitrary PDAGs by
\citet{WBL2021}, which utilizes the topological ordering of a consistent
extension (\textsc{ce-meek}). The input PDAGs are the same as in
Section~\ref{sec:eval}, and further results for scale-free PDAGs are provided
in Appendix~\ref{appendix:further-eval-mo}.

The run times of the two approaches on random PDAGs are presented in
Fig.~\ref{fig:results-mo-pdag-er}.
Not surprisingly, \textsc{ce-meek} outperforms \textsc{direct-meek} on
every input graph.
While the difference in their run times is relatively small
on sparse input graphs (left), it increases drastically as $n$ grows
on denser input graphs (right)%
\footnote{We use \textsc{dtic} to compute the
  consistent extension. For sparse graphs, the use
  of \textsc{wbl} can give a significant speedup,
  which would in turn translate to \textsc{ce-meek}.}.
\begin{figure}
  \centering
  \resizebox{\textwidth}{!}{\input{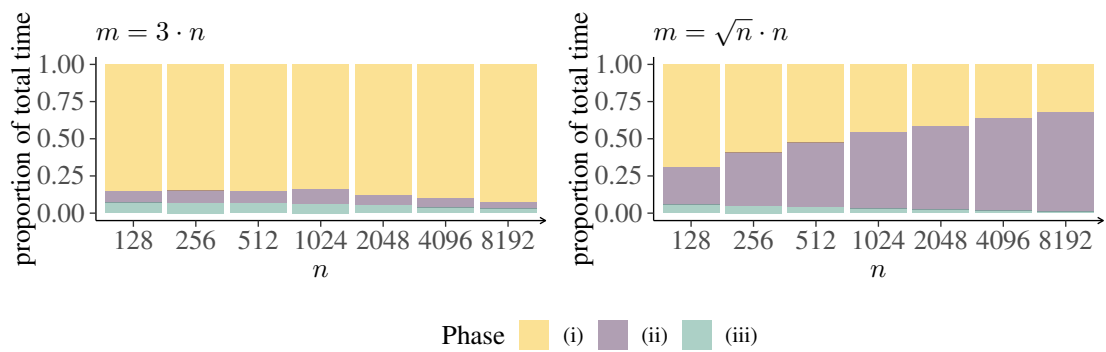}}
  \caption{Proportions of the total run time of \textsc{ce-meek}
    for the three phases on randomly generated PDAGs with
    $m = 3 \cdot n$ edges (left) and $m = \sqrt n \cdot n$ edges (right).}
  \label{fig:results-mo-perc-pdag-er}
\end{figure}
In addition to the plain comparison of run times, we also analyze the
time spent by \textsc{ce-meek} on the different phases of the algorithm.
More specifically, we measure the time for (i) computing a consistent
extension, (ii) finding the corresponding CPDAG, and (iii) applying the
Meek rules in a single iteration over the vertices.
The results for the same input graphs as before are depicted in
Fig.~\ref{fig:results-mo-perc-pdag-er}.
Each bar is divided into the proportions of the total run time for the
three phases, i.\,e., adding together the proportions of the three phases
equals 100 percent of the run time.
For sparse graphs (left), the majority of the time (more than 75
percent) is spent on phase~(i), demonstrating that in
order to obtain fast algorithms to compute maximal orientations, it is
crucial to have algorithms solving the extension problem efficiently.
Even though the proportion of phase~(ii) dominates for dense graphs (right)
with many vertices, phase~(i) still accounts for a significant portion
(at least 25 percent) of the total run time.

\section{Conclusions} \label{sec:conclusion}
In this paper, we demonstrate the effectiveness of utilizing consistent
extensions for the task of maximally orienting a PDAG. We started by
revisiting the extension problem, presenting two new approaches to efficiently compute
consistent DAG extensions in practical applications.
The first approach (\textsc{dth}) refines the widespread Dor-Tarsi algorithm by
the employment of a simple heuristic, which reduces the computational
effort by prioritizing low-cost vertices throughout its iterations.
The second approach (\textsc{dtic}) stores additional information between iterations
to avoid duplicate iterations and thereby matches the worst-case
complexity of the \textsc{wbl} algorithm, which achieves the conditional
lower bound of $O(n^3)$ for the extension problem.
In a practical evaluation, we show that \textsc{dtic} exhibits the most stable
performance overall, combining the advantages of the other approaches.
%By also applying the heuristic from \textsc{dth}, \textsc{dtic} combines the
%advantages of (i) not relying on a costly initialization step, (ii) allowing
%for the early cancellation of loops, and (iii) having worst-case complexity
%$O(n^3)$. 
Based on those insights and results, we highlight an important
application -- the maximal orientation of PDAGs, a procedure ubiquitous in causal
discovery. We demonstrate experimentally that utilizing consistent extensions yields
highly reliable and effective algorithms for this task, which outperform the
direct use of the Meek rules currently used most commonly.

\acks{The research of Malte Luttermann was partly supported by the Medical
Cause and Effects Analysis (MCEA) project.}

\bibliography{main.bib}

\appendix

\section{Comparison of Graph Representations} \label{appendix:graph-repr}
Choosing an appropriate graph representation
is crucial to obtain effective extension algorithms.
In particular, two important operations that need to be executed fast to
solve the extension problem efficiently are adjacency tests and the
removal of vertices and edges from the graph.
In this section, we compare two different graph representations to support
carrying out these operations efficiently.
The first representation utilizes a hashed data structure to represent adjacencies
in a graph and a the second one employs a combination of linked lists and an
adjacency matrix for this purpose.

The graph representation used throughout the course of this paper keeps
track of three sets storing the neighbors (ingoing, outgoing, and undirected)
for each vertex in the graph.
To access elements in these sets in expected constant time, hashing is applied.
Representing the neighbors of every vertex with the help of hash sets leads
to a simple, yet practical graph representation.
However, iterating over hash sets (e.\,g., during the potential-sink check)
is suboptimal, as the entire hash table has to be iterated over.
Furthermore, as outlined in Section~\ref{sec:dtic}, a representation based on
linked lists allows \textsc{dtic} to exit its loops earlier,
that is, as soon as a neighbor violating the potential-sink property is
found, which is not possible if neighbors are stored in a hashed data structure.

To analyze the impact of these aspects on the performance of \textsc{dtic}
and the other algorithms, we also implemented a graph representation using an
adjacency matrix and linked lists to store adjacencies in the graph, which was
proposed by~\citep{WBL2021}.
More precisely, there are three linked lists to store the ingoing, outgoing,
and undirected neighbors, respectively, for each vertex. The adjacency
matrix contains a pointer to the corresponding linked list entry for every
edge in the graph.
Adjacency tests run in constant time on the adjacency matrix and the removal
of vertices from the graph runs in constant time as well, as the
corresponding pointer in the matrix enables an access to any neighbor in the
linked list in $O(1)$. 
The main drawback of this approach is clearly the large
memory requirement of the $O(n^2)$ adjacency matrix. Allocating the
corresponding memory is also time-consuming.

We found that the usage of linked lists instead of hash sets does not
provide significant improvements for the run times of the extension algorithms.
In most of the evaluated scenarios, the contrary is the case.
A direct comparison of implementations using hash sets and linked lists can be found in
Table~\ref{table:hashing-vs-ll-pdag-er} where the input graphs and
settings are the same as in Section~\ref{sec:eval}.

\begin{table}[t]
  \centering
  \begin{tabular}{l|llllllll} \toprule
    $n$ & \textsc{dt} & \textsc{dt-ll} & \textsc{dth} & \textsc{dth-ll} & \textsc{dtic} & \textsc{dtic-ll} & \textsc{wbl} & \textsc{wbl-ll} \\ \midrule
    $128$  & 0.83   & 0.14   & 0.88   & 0.19   & 0.94   & 0.20   & 1.29   & 0.26   \\
    $256$  & 1.80   & 0.41   & 1.89   & 0.54   & 1.97   & 0.58   & 2.60   & 0.68   \\
    $512$  & 3.89   & 1.60   & 4.07   & 1.88   & 4.13   & 2.03   & 5.21   & 2.15   \\
    $1024$ & 9.96   & 6.76   & 10.19  & 7.62   & 9.72   & 8.28   & 10.88  & 9.53   \\
    $2048$ & 27.01  & 39.55  & 25.66  & 40.40  & 23.32  & 42.86  & 23.35  & 50.90  \\
    $4096$ & 87.13  & 149.06 & 85.38  & 190.61 & 68.50  & 201.69 & 51.98  & 236.52 \\
    $8192$ & 370.79 & 606.02 & 323.64 & 608.67 & 222.52 & 654.45 & 116.92 & 871.23 \\ \bottomrule
  \end{tabular}
  \hfill\\[0.5cm]
  \begin{tabular}{l|llllllll} \toprule
    $n$ & \textsc{dt} & \textsc{dt-ll} & \textsc{dth} & \textsc{dth-ll} & \textsc{dtic} & \textsc{dtic-ll} & \textsc{wbl} & \textsc{wbl-ll} \\ \midrule
    $128$  & 0.87   & 0.19   & 0.97   & 0.28   & 1.04   & 0.29   & 2.35   & 0.46    \\
    $256$  & 1.86   & 0.55   & 2.11   & 0.81   & 2.23   & 0.84   & 5.61   & 1.21    \\
    $512$  & 3.92   & 1.97   & 4.71   & 2.86   & 4.93   & 2.98   & 13.09  & 4.30    \\
    $1024$ & 9.62   & 9.14   & 11.90  & 11.68  & 12.14  & 12.08  & 30.98  & 15.93   \\
    $2048$ & 23.60  & 48.02  & 30.77  & 54.55  & 31.38  & 56.30  & 75.68  & 72.03   \\
    $4096$ & 72.80  & 175.78 & 93.31  & 236.88 & 96.00  & 244.46 & 209.80 & 323.04  \\
    $8192$ & 269.53 & 685.58 & 333.44 & 752.79 & 336.09 & 785.27 & 594.51 & 1159.83 \\ \bottomrule
  \end{tabular}
  \caption{Comparison of the algorithms \textsc{dt}, \textsc{dth},
    \textsc{dtic}, and \textsc{wbl} with a linked lists implementation (\textsc{-ll})
    on random PDAGs with $m = 3 \cdot n$ edges (above) and $m = \log_2(n) \cdot n$
    edges (below).}
  \label{table:hashing-vs-ll-pdag-er}
\end{table}

Clearly, the usage of linked lists appears to be at a disadvantage for the
instances considered in this work.
However, we remark that, in some cases, the usage of linked lists is
beneficial. For example, on chordal graphs (see
Appendix~\ref{appendix:further-eval-ext} for more details),
\textsc{dt} is faster on graphs with a linked list implementation
compared to graphs using a hashed data structure.
In Appendix~\ref{appendix:further-eval-ext}, we also introduce scale-free PDAGs, for which
the comparison of hash sets and linked lists yields the same relations
as in Table~\ref{table:hashing-vs-ll-pdag-er}.
In conclusion, our experiments demonstrate that the linked list
representation introduces overhead, which does not pay off for the generally
sparse graphs considered in this work and usually occurring in practice.

\section{Further Experimental Results for Extendability} \label{appendix:further-eval-ext}
To complement the experimental results for randomly generated PDAGs
presented in Section~\ref{sec:eval}, we evaluate the algorithms
\textsc{dt}, \textsc{dth}, \textsc{dtic}, and \textsc{wbl} on scale-free
PDAGs and chordal graphs.
Scale-free graphs are graphs whose degree distribution is a
power law distribution, i.\,e., there are few vertices with a high degree
while most of the vertices have a rather small degree.
The scale-free PDAGs are generated in a similar fashion as the graphs
from Section~\ref{sec:eval}, that is, the general procedure of
(i) creating a random DAG $D$, (ii) replacing all directed edges not
participating in a v-structure by an undirected edge, and (iii) orienting
between two and five (randomly chosen) undirected edges according to $D$
stays the same but the initial DAG $D$ is generated in a different way.
More precisely, the initial DAG is generated by first creating a random
undirected scale-free graph using the Barabási-Albert
model~\citep{BarabasiAlbert2002} and afterward constituting a
random permutation of its vertices as a topological ordering according
to which the edges are then directed.
All parameter choices are identical to those in Section~\ref{sec:eval}.

To generate the chordal graphs, we apply the random subtree intersection
method introduced by \citet{Seker2017}.
Chordal graphs provide an interesting addition to PDAGs as they are fully
undirected and extendable by definition, thus demanding quite some
computational effort to extend the graph.
We set $n = 128, 256, \dots, 8192$ again and use $k = 3, 5, \log_2(n), \sqrt n$
as a parameter for the random subtree intersection method.
The parameter $k$ determines the average size of the random subtrees used
to generate the chordal graph and thus controls the number of edges in the graph.
However, the exact number of edges slightly differs between instances.

\begin{figure}[t]
  \centering
  \resizebox{\textwidth}{!}{\input{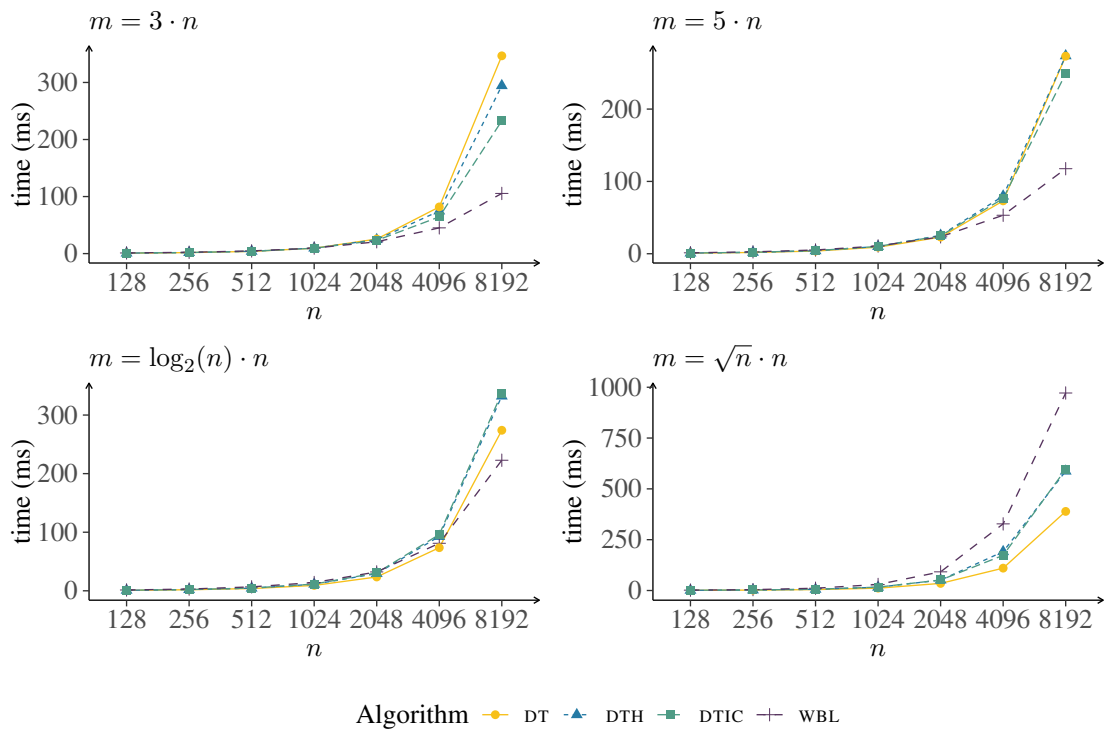}}
  \caption{Run times of the algorithms \textsc{dt}, \textsc{dth},
    \textsc{dtic}, and \textsc{wbl} on randomly generated scale-free PDAGs 
     of $n$ vertices and $m$ edges,  
    with $m = 3 \cdot n$  (top left), $m = 5 \cdot n$  (top right),
    $m = \log_2(n) \cdot n$  (bottom left), and $m = \sqrt n \cdot n$
    (bottom right).}
    \label{fig:results-ext-pdag-ba}
\end{figure}

The results for scale-free PDAGs can be found in
Fig.~\ref{fig:results-ext-pdag-ba} where we observe the same patterns as in
Fig.~\ref{fig:results-ext-pdag-er}, i.\,e., \textsc{wbl} is the fastest for
sparser graphs ($m = 3 \cdot n, 5 \cdot n, \log_2(n) \cdot n$) but is the
slowest on denser graphs ($m = \sqrt n \cdot n$) while the opposite holds
for \textsc{dt} (that is, \textsc{dt} is the slowest on sparse graphs with
$m = 3 \cdot n$ and the fastest on dense graphs with $m = \sqrt n \cdot n$).
In total, \textsc{dtic} yields a stable performance again, showing that
\textsc{dtic} retains the practicality of \textsc{dt} while maintaining
the theoretical bounds on the run time complexity of \textsc{wbl}.

\begin{figure}[t]
  \centering
  \resizebox{\textwidth}{!}{\input{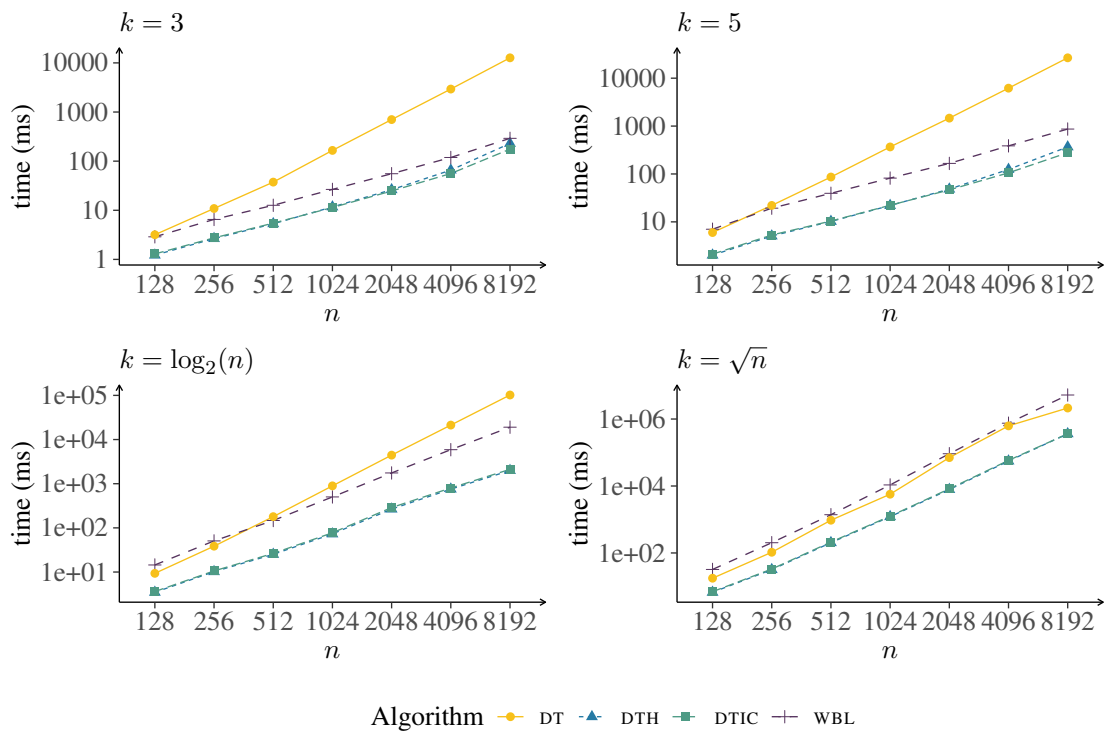}}
  \caption{Run times of the algorithms \textsc{dt}, \textsc{dth},
    \textsc{dtic}, and \textsc{wbl} on randomly generated chordal graphs
    with $k = 3$ (top left), $k = 5$ (top right), $k = \log_2(n)$
    (bottom left), and $k = \sqrt n$ (bottom right) with all y-axes
    being log-scaled.}
  \label{fig:results-ext-chordal}
\end{figure}

Further, Fig.~\ref{fig:results-ext-chordal} displays the run times of
\textsc{dt}, \textsc{dth}, \textsc{dtic}, and \textsc{wbl} on chordal graphs.
Note that the y-axes are log-scaled in all plots.
\textsc{dt} is inferior to all other algorithms sparse graphs (top left and
top right) and also becoming the slowest on slightly denser graphs (bottom
left).
Eventually, the run time of \textsc{dt} becomes almost 100 times
slower than the run times of the other algorithms, for example at $n = 8192$
and $k = 3$ (among others).
\textsc{wbl} handles sparse graphs (top left and top right) roughly as good
as \textsc{dth} and \textsc{dtic} but cannot keep pace for the denser graphs
shown at the bottom plots.
The results indicate that \textsc{dth} and \textsc{dtic} yield the best
performance on graphs where more effort is necessary to compute a consistent
extension.

\section{Further Experimental Results for Maximal Orientations} \label{appendix:further-eval-mo}
For the sake of completeness, we also give further experimental results
for the algorithms \textsc{direct-meek} and \textsc{ce-meek} in addition
to the results presented in Section~\ref{sec:appl-meek-eval}.
We report the results for scale-free PDAGs which are identical to those
from the previous section and add more edge densities to the evaluated
scenarios from Section~\ref{sec:appl-meek-eval}\footnote{Chordal graphs are
obviously not interesting in the setting of maximal orientations as they
contain only undirected edges and hence no Meek rule is applicable on them.}.

\begin{figure}[t]
  \centering
  \resizebox{\textwidth}{!}{\input{plots/results-mo-pdag-er-more-m.tex}}
  \caption{Run times of the algorithms \textsc{direct-meek} and \textsc{ce-meek}
    on randomly generated PDAGs with $m = 5 \cdot n$ edges (left) and
    $m = \log_2(n) \cdot n$ edges (right).}
  \label{fig:results-mo-pdag-er-more-m}
\end{figure}

\begin{figure}[t]
  \centering
  \resizebox{\textwidth}{!}{\input{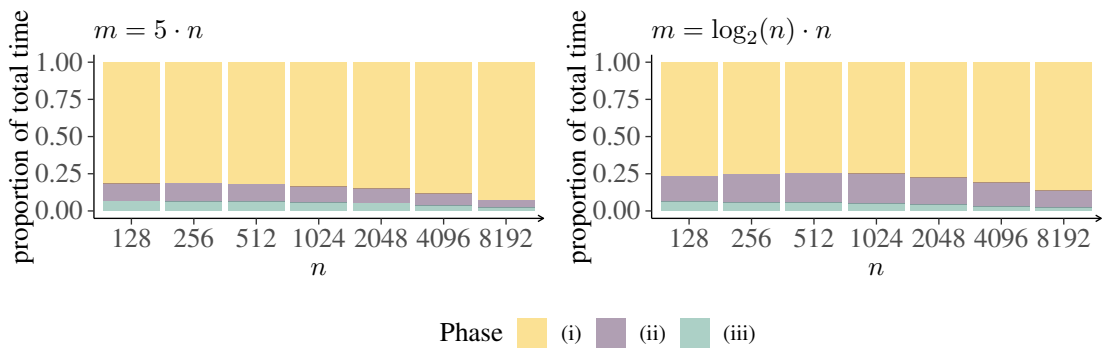}}
  \caption{Proportions of the total run time of \textsc{ce-meek}
    for the three phases on randomly generated PDAGs with
    $m = 5 \cdot n$ edges (left) and $m = \log_2(n) \cdot n$ edges (right).}
  \label{fig:results-mo-perc-pdag-er-more-m}
\end{figure}

Fig.~\ref{fig:results-mo-pdag-er-more-m} depicts the run times of
\textsc{direct-meek} and \textsc{ce-meek} on randomly generated PDAGs
containing $m = 5 \cdot n$ and $m = \log_2(n) \cdot n$ edges.
These PDAGs are generated the same way as the PDAGs in
Fig.~\ref{fig:results-mo-pdag-er}.
The run times pictured in Fig.~\ref{fig:results-mo-pdag-er-more-m} show the
same pattern as in Fig.~\ref{fig:results-mo-pdag-er}, namely that
\textsc{ce-meek} is faster than \textsc{direct-meek} in every scenario and
the advantage of \textsc{ce-meek} increases with an increasing number of
vertices and edges in the input graph, verifying that the use of consistent
extensions for the computation of maximal orientations is highly beneficial.

The time spent by \textsc{ce-meek} on the different phases of the algorithm
for PDAGs with $m = 5 \cdot n$ and $m = \log_2(n) \cdot n$ edges is plotted in
Fig.~\ref{fig:results-mo-perc-pdag-er-more-m}.
While the left plot is similar to the left plot from
Fig.~\ref{fig:results-mo-perc-pdag-er}, the right plot in
Fig.~\ref{fig:results-mo-perc-pdag-er-more-m} exhibits a greater proportion
of the total run time needed for phase~(i) (that is, computing a consistent
extension) than the right plot in Fig.~\ref{fig:results-mo-perc-pdag-er},
showing that phase~(i) is dominant for sparser graphs and phase~(ii) requires
the majority of the total run time for dense graphs ($m = \sqrt n \cdot n$,
right plot in Fig.~\ref{fig:results-mo-perc-pdag-er}).

\begin{figure}[t]
  \centering
  \resizebox{\textwidth}{!}{\input{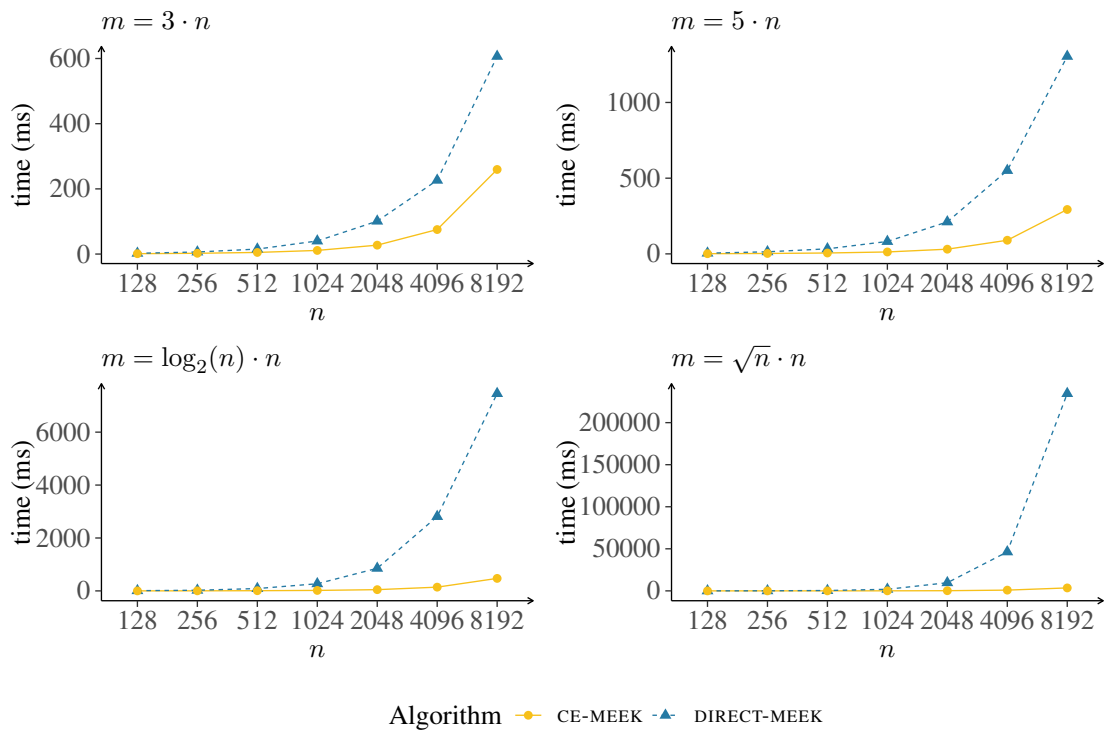}}
  \caption{Run times of the algorithms \textsc{direct-meek} and \textsc{ce-meek}
    on randomly generated scale-free PDAGs 
    of $n$ vertices and $m$ edges,  
    with $m = 3 \cdot n$  (top left), $m = 5 \cdot n$  (top right),
    $m = \log_2(n) \cdot n$  (bottom left), and $m = \sqrt n \cdot n$
    (bottom right).}
  \label{fig:results-mo-pdag-ba}
\end{figure}

Fig.~\ref{fig:results-mo-pdag-ba} presents the run times of \textsc{direct-meek}
and \textsc{ce-meek} on scale-free PDAGs.
As in Fig.~\ref{fig:results-mo-pdag-er} and
Fig.~\ref{fig:results-mo-pdag-er-more-m}, \textsc{ce-meek} is superior to
\textsc{direct-meek} on all input graphs.
Overall, handling scale-free PDAGs takes more computational effort compared
to handling PDAGs having their edges distributed at random.

\begin{figure}[t]
  \centering
  \resizebox{\textwidth}{!}{\input{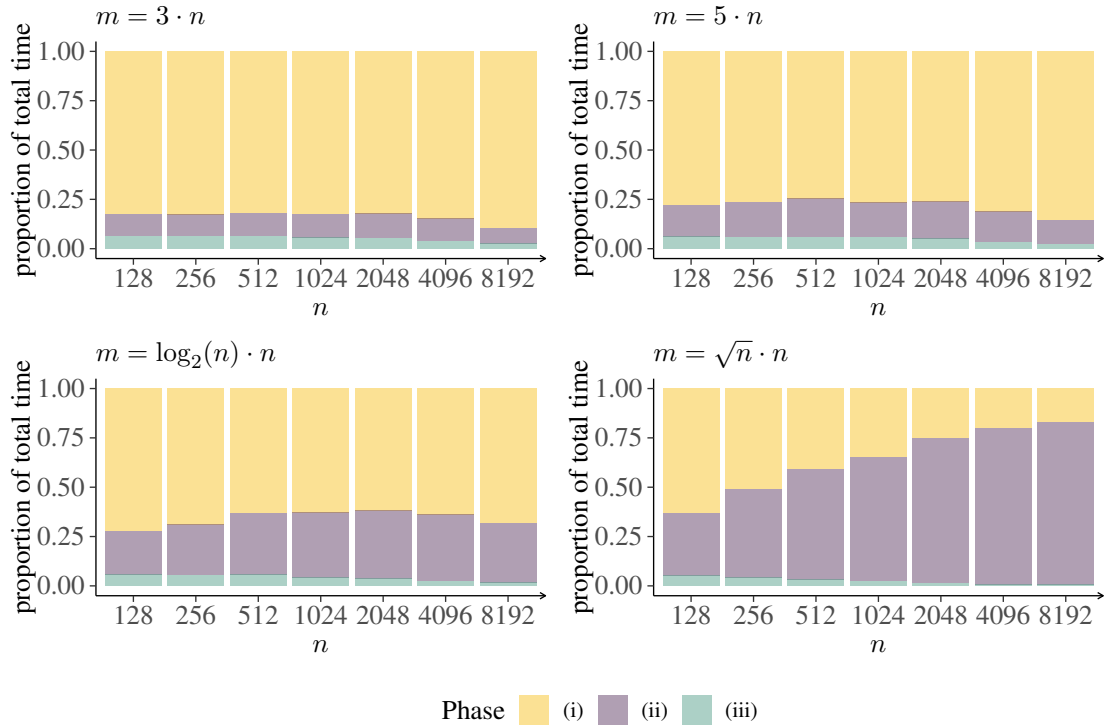}}
  \caption{Proportions of the total run time of \textsc{ce-meek}
    for the three phases on randomly generated scale-free PDAGs 
    of $n$ vertices and $m$ edges,  
    with $m = 3 \cdot n$  (top left), $m = 5 \cdot n$  (top right),
    $m = \log_2(n) \cdot n$  (bottom left), and $m = \sqrt n \cdot n$
    (bottom right).}
  \label{fig:results-mo-perc-pdag-ba}
\end{figure}

Furthermore, a visualization of the time spent by \textsc{ce-meek} on the
different phases of the algorithm for scale-free input PDAGs is given in
Fig.~\ref{fig:results-mo-perc-pdag-ba}.
The plots for $m = 3 \cdot n$ (top left) and $m = \sqrt n \cdot n$ (bottom right)
are similar to the plots in Fig.~\ref{fig:results-mo-perc-pdag-er} and the plots
for $m = 5 \cdot n$ (top right) and $m = \log_2(n) \cdot n$ (bottom left) are
similar to the plots in Fig.~\ref{fig:results-mo-perc-pdag-er-more-m}.
Despite the similarities, we observe greater proportions of phase~(ii) (i.\,e.,
finding the corresponding CPDAG to the DAG computed in the first phase) in
Fig.~\ref{fig:results-mo-perc-pdag-ba} than in Fig.~\ref{fig:results-mo-perc-pdag-er}
and Fig.~\ref{fig:results-mo-perc-pdag-er-more-m}, showing that phase~(ii)
requires more effort on scale-free PDAGs than on PDAGs having their edges
distributed at random.
It becomes evident again that higher graph densities increase the proportion of
phase~(ii) of the total run time.
\end{document}